%% file: main.tex
\documentclass{article}

\usepackage[english]{babel}
\usepackage[utf8]{inputenc}
\usepackage{amsmath}
\usepackage{graphicx}

\usepackage{amsfonts}
\usepackage[colorlinks=true,allcolors=blue]{hyperref}%
\usepackage{url}
\usepackage{booktabs}
\usepackage{multirow}
\usepackage{graphicx}
\usepackage{subcaption}
\usepackage{wrapfig}
\usepackage{amsthm}
\usepackage[normalem]{ulem}
\usepackage{mathrsfs}

\newtheorem{theorem}{Theorem}

\newtheorem{remark}{Remark}

\newtheorem{proposition}[theorem]{Proposition}

\title{Task-Agnostic Robust Representation Learning}

\author{A. Tuan Nguyen\thanks{Corresponding author: A. Tuan Nguyen, \texttt{tuan@robots.ox.ac.uk}}\;\;\thanks{University of Oxford} \and  Ser Nam Lim\thanks{Meta AI Research} \and Philip Torr\footnotemark[2]}

\date{March 9, 2022}

\begin{document}
\maketitle

\begin{abstract}
It has been reported that deep learning models are extremely vulnerable to small but intentionally chosen perturbations of its input. In particular, a deep network, despite its near-optimal accuracy on the clean images, often mis-classifies an image with a worst-case but humanly imperceptible perturbation (so-called adversarial examples). To tackle this problem, a great amount of research has been done to study the training procedure of a network to improve its robustness. However, most of the research so far has focused on the case of supervised learning. With the increasing popularity of self-supervised learning methods, it is also important to study and improve the robustness of their resulting representation on the downstream tasks. In this paper, we study the problem of robust representation learning with unlabeled data in a task-agnostic manner. Specifically, we first derive an upper bound on the adversarial loss of a prediction model (which is based on the learned representation) on any downstream task, using its loss on the clean data and a robustness regularizer. Moreover, the regularizer is task-independent, thus we propose to minimize it directly during the representation learning phase to make the downstream prediction model more robust. Extensive experiments show that our method achieves preferable adversarial performance compared to relevant baselines.
\end{abstract}

\input{sections/introduction}

\input{sections/related_works}
\input{sections/approach}
\input{sections/experiments}
\input{sections/conclusion}

\bibliographystyle{unsrt}
\bibliography{sections/citations}

\input{sections/appendix}

\end{document}

%% file: sections/introduction.tex
\section{Introduction}
Deep learning has achieved state-of-the-art performance in many tasks such as image classification, object detection, and natural language processing. Instead of having to select handcrafted features and representation of the input as in classical machine learning, deep learning has the ability to automatically learn a meaningful representation with deep networks and gradient descent. However, the success of deep learning relies on the availability of a large amount of labeled data, which is expensive in practice. Self-supervised learning \cite{chen2020simple,grill2020bootstrap,chen2021exploring,zbontar2021barlow} has gained interest as a solution to the above problem due to its ability to learn from unlabeled data. However, although the representation learned via self-supervised learning is often meaningful for downstream tasks, the resulting prediction model usually lacks adversarial robustness. Currently, to the best of our knowledge, virtually all existing adversarial training methods \cite{madry2017towards,zhang2019theoretically} require labels and/or the prediction task. Furthermore, for computational reasons, we do not want to re-do the expensive adversarial training for the downstream tasks (this is also in the spirit of self-supervised learning); therefore, the (theoretical) transferability of a robust representation among different tasks is also crucial. However, this transferability aspect of robustness has not been well-studied. For example, \cite{kim2020adversarial} proposes a framework for adversarial contrastive learning, which enforces the robustness of the representation network by finding worst-case adversaries that maximize the contrastive loss, followed by minimizing that loss with respect to the network parameters. However, it is not clear how this robustness can be transferred to a downstream task.

\begin{figure*}[t!]
	\centering
	\includegraphics[width=0.9\linewidth]{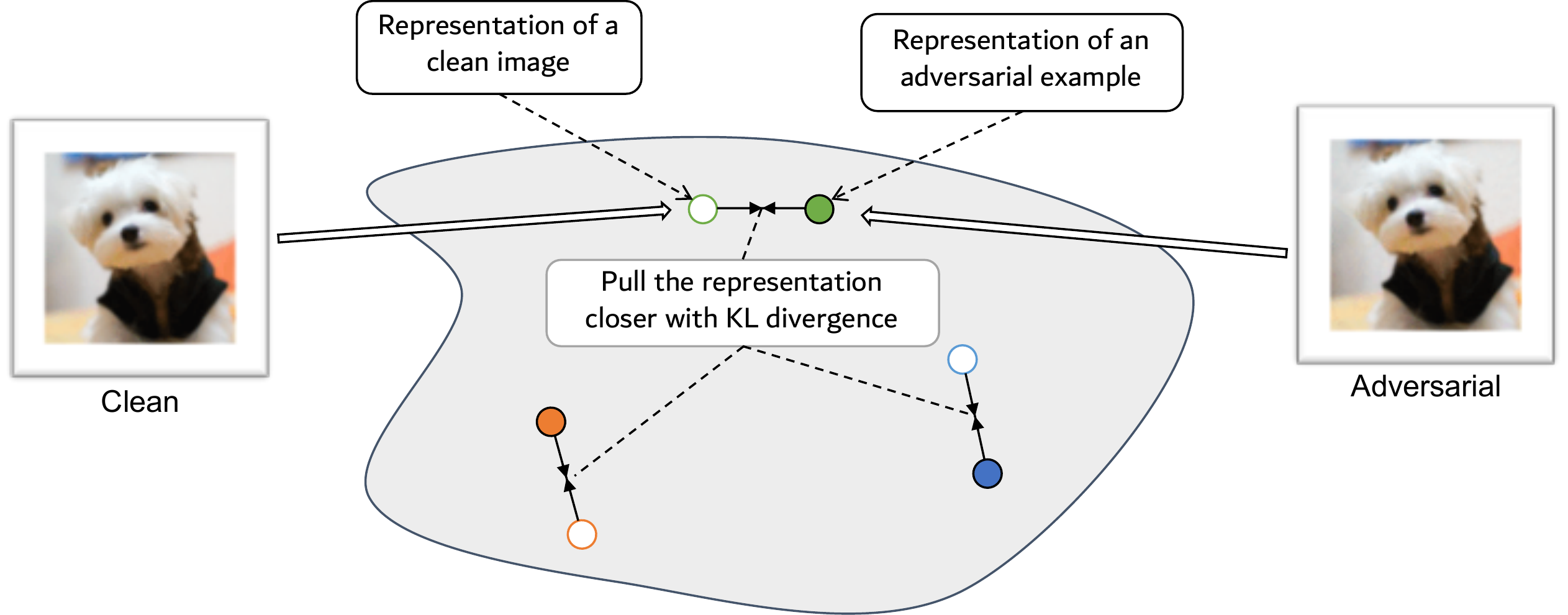}
	\vspace{-0.1in}
	\caption{\textbf{Illustration of our proposed regularizer}. This figure illustrates the (probabilistic) representation space, with each circle representing (the distribution of) the representation of an image. Non-filled circles depict the (probabilistic) representations of natural images and filled circles depict that of their adversarial examples. We use the KL divergence to pull the representations of an image and its adversarial images closer, which improves the bound of the adversarial loss on a downstream task.}
	\label{illustration}
\end{figure*}

\begin{figure}[t!]
        \centering
        \begin{subfigure}[b]{0.5\linewidth}
                \centering
                \includegraphics[width=\textwidth]{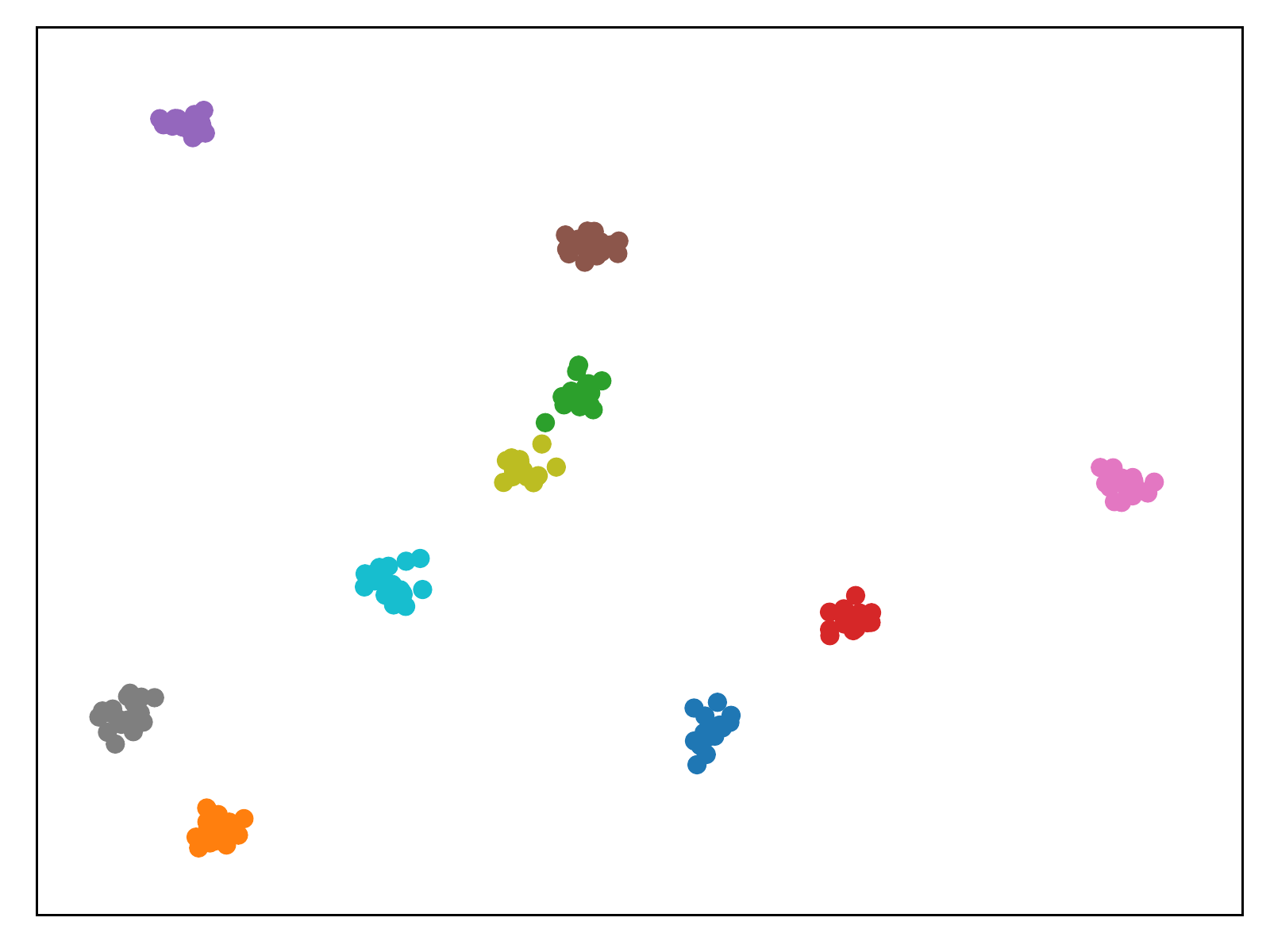}
        \end{subfigure}
        \hfill 
        \begin{subfigure}[b]{0.5\linewidth}
                \centering
                \includegraphics[width=\linewidth]{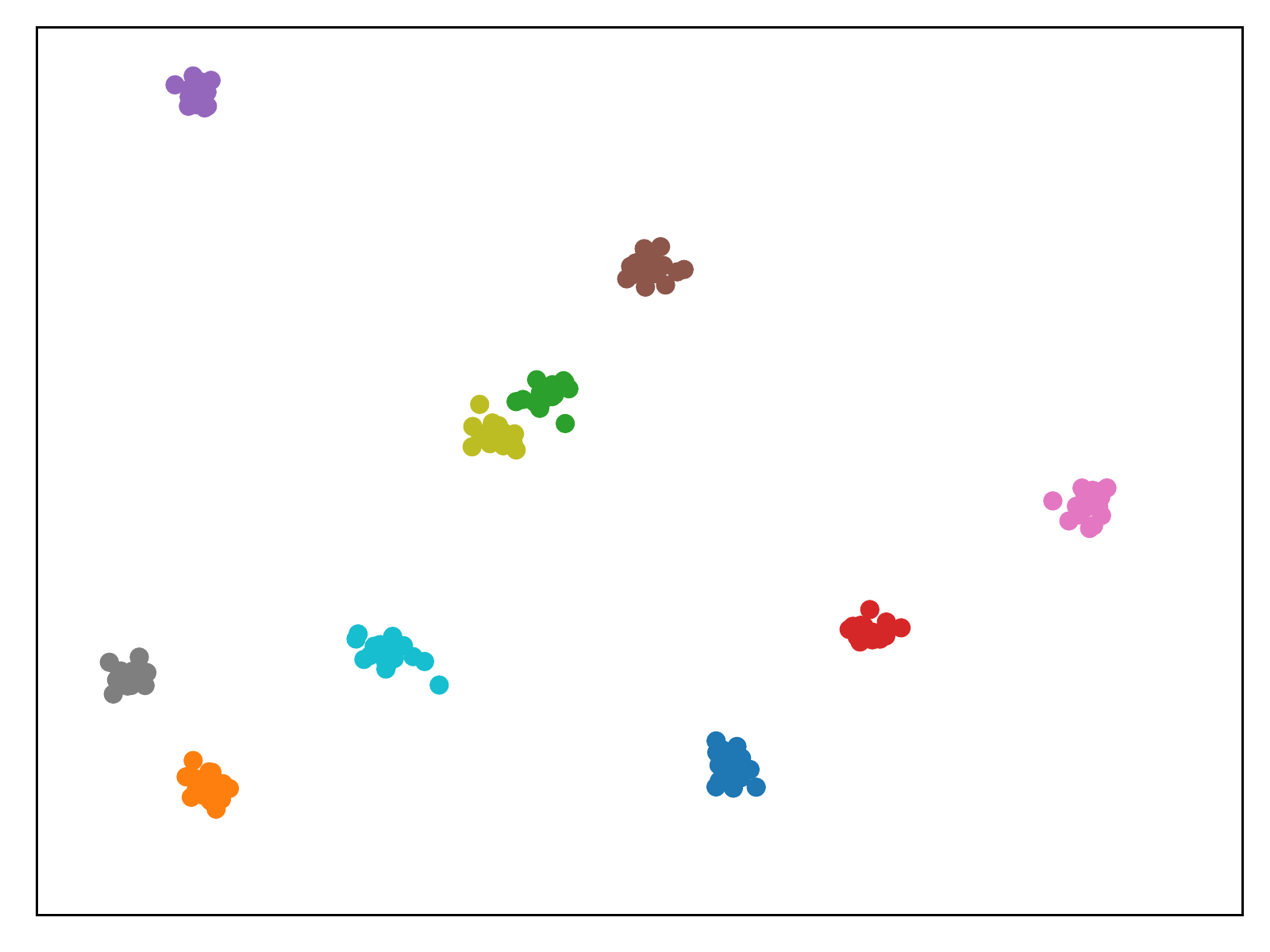}
        \end{subfigure}
        \vspace{-0.2in}
        \caption{\textbf{Visualization of the probablistic representation of our method}. Each color corresponds to a single image. For each image $x$, we sample 20 $z$'s from the probabilistic representation distribution $p_\theta(z|x)$ (hence the clusters of points). The left figure is for the original images and the right figure is for the unsupervised adversaries. Our method enforces the representation distribution of an image to be close to that of its adversaries, making the downstream models more robust.} \label{visualization}
\end{figure}

In this paper, we develop a task-independent robust representation learning method to tackle the above issue. Intuitively, if we could enforce the representation $z$ of an input $x$ to be close to that of its neighbors in the adversarial ball (for example, a $l_\infty$ ball around $x$), then it would make the downstream model more robust (Figure~\ref{illustration}). In particular, any adversarial example $x^{adv}$ of $x$ produced by an attack on the downstream task is within this adversarial ball, thus its representation would be close to that of $x$, making it harder for the task's decision boundary to separate them. However, there are some challenges when applying this idea naively. One might think that we can find a worst-case adversarial example $x^{u\_adv}$ (unsupervised adversary, as opposed to the typical supervised adversary $x^{adv}$) with the largest distance to $x$ in the representation space (based on some distance metrics), and minimize that distance with respect to the network's parameters. This, however, is problematic for a typical deterministic representation network, since it is not trivial how to define the distance on the representation space (for example, using the $l_2$ distance between the representations would not help since the model can ``cheat'' by making the norm of the representation smaller, which can be easily compensated by making the weights of the next fully connected layer bigger). Fortunately, in this paper, we observe that with a probabilistic representation network $p_\theta(z|x)$, we can define the ``closeness'' of the representations by the KL divergence, which allows for creating adversarial examples during unsupervised/self-supervised learning via an inner maximization of the KL term ($\text{KL}[p_\theta(z|x)|p_\theta(z|x^{u\_adv})]$). Visualization of the representation distribution of an image and its adversarial example of our method can be found in Figure~\ref{visualization}. We also show that this leads to a robustness regularizer for unsupervised learning that is \textbf{provably transferable} to downstream tasks (more details in the Approach section). Note that this idea can also be thought of as enforcing a low Lipschitzness of the representation network. However, as mentioned above, a typical norm-based (e.g., $l_1$ and $l_2$) Lipschitzness in the representation space is not meaningful (due to the aforementioned ``cheat''), whereas using KL divergence solves the problem and also leads to a theoretical guarantee on the downstream robustness.

In this work, we first propose an upper bound of the adversarial loss of a model for a certain prediction task based on its loss on the clean data and a robustness regularizer. The regularizer, which is based on the KL divergence as described above, only depends on the representation mapping and is independent of the task. Therefore, we can minimize this regularizer term during the representation learning phase (of an unsupervised or self-supervised method), so that the robustness will be transferred directly to a prediction model when the representation is used for a downstream task. Our method can be straightforwardly applied to any representation learning method, including unsupervised learning (e.g., VAE) or any self-supervised learning model.

Our contributions in this work are threefold:
\begin{itemize}
    \item We derive an upper bound on the adversarial loss of a prediction model on a certain task, based on a task-independent robustness regularizer. 
    \item We propose to incorporate the above regularizer into existing representation learning frameworks to improve the adversarial robustness of the representation on downstream tasks.
    \item We demonstrate the use of our proposed regularizer with existing representation learning frameworks, and show that the resulting models achieve SOTA results in the unsupervised robust representation learning task.
\end{itemize}

%% file: sections/related_works.tex
\section{Related Works}
\vspace{-0.1in}
\subsection{Adversarial Training}
It has been shown that neural networks, even with a high classification accuracy, are vulnerable to small bounded (but intentionally worst-case chosen) adversarial perturbations \cite{szegedy2013intriguing}. From this observation, many methods have been proposed to alter the typical training procedure of a neural network to improve its adversarial robustness (hence the name adversarial training). Two of the most often-used methods are AT \cite{madry2017towards} and TRADES \cite{zhang2019theoretically}, in which the algorithms try to find a worst-case perturbation of the input (called an adversarial example) with an inner maximization problem and minimize the loss (or a regularizer) with respect to that adversarial example. Studying the inner maximization problem is also an active and attractive research direction, with the aim to develop fast and/or accurate methods to find the adversarial examples. For example, PGD \cite{madry2017towards} uses multiple steps of projected gradient descent, which is accurate (and somewhat ``gold-standard'') but expensive. For this computational reason, many one-step algorithms have been proposed, including FGSM \cite{goodfellow2014explaining}, RS-FGSM \cite{wong2020fast} and GradAlign \cite{andriushchenko2020understanding}, with varying level of success. These adversarial defense/attack methods are relevant to our work since we also need to solve an inner maximization problem to find (unsupervised) adversarial examples for the robustness regularizer.
\vspace{-0.1in}
\subsection{Self-Supervised Learning}
As mentioned earlier, self-supervised learning has received great interest due to its ability to learn from unlabeled data, which reduces the need for expensive annotations of images. Self-supervised learning is based on a user-defined pretext task, which can be as simple as to predict the rotation angle of an image or more complex such as to solve a Jigsaw puzzle \cite{noroozi2016unsupervised}. Recently, a popular and successful self-supervised learning paradigm is to learn representations that are invariant under different
augmentations (also referred to as ‘distortions’) of an image \cite{chen2020simple,zbontar2021barlow,chen2021exploring}. The idea of this learning paradigm is to maximize the similarity between representations of two augmentations of an image, while avoiding network collapse (to a trivial and meaningless solution such as a constant function) by different objective functions. Since our proposed robustness regularizer is task-independent, it can be straightforwardly applied to most of these self-supervised learning methods.

\subsection{Robust Self-Supervised Learning}
Recently, the field of unsupervised robust representation learning has gained increasing interest, with the goal to leverage unlabeled data and learn a robust and meaningful representation for downstream tasks. One of the main baselines to our work is RoCL \cite{kim2020adversarial}, which applies adversarial training directly on the contrastive loss in SimCLR. However, it is not clear if a representation that is robust to the contrastive loss will be robust to the prediction loss of a downstream task. Some other methods \cite{jiang2020robust,hendrycks2019using} perform adversarial training on the self-supervised task to aid the adversarial training of the main task (not to replace), which are less related to our work. Similarly, \cite{alayrac2019labels} utilizes unlabeled data to help train a robust classifier in a semi-supervised manner. Meanwhile, \cite{awasthi2021adversarially} propose an algorithm to improve the robustness of PCA. However, PCA is not a common component of modern deep learning architectures/pipelines, thus its practicality might be limited.

%% file: sections/approach.tex
\vspace{-0.1in}
\section{Approach}
\vspace{-0.1in}
\subsection{Problem statement}

Assume that we have a distribution $p(x)$ of data (e.g., images), with $x \in \mathcal{X}$ as the input. We are interested in the problem of robust representation learning, where we want to learn a representation $z$ of $x$ with the mapping $p_\theta(z|x)$, parameterized by $\theta$; so that $z$ is meaningful for downstream tasks (to be defined below), and that any classifier (based on $z$) of the downstream tasks should be robust against adversarial attacks. For computational reasons, we do not want to re-do the adversarial training, and desire that the robustness transfer directly to the downstream tasks. The representation $z$ can be probabilistic (e.g., $p_\theta(z|x)=\mathcal{N}(z;\mu_\theta(x),\sigma_\theta(x))$) or deterministic (i.e., $p_\theta(z|x)=\delta_{g_\theta(x)}(z)$ with a deterministic function $g_\theta$). In this paper, we will especially consider a probabilistic representation mapping. In practice, the representation can be learn by unsupervised learning methods (e.g., VAE) or self-supervised learning methods (with a pretext task).

\begin{remark}
    A note on the choice of the representation distribution $p_\theta(z|x)$.
    \label{gaussian}
\end{remark}
\vspace{-0.1in}
\quad As mentioned earlier, in this paper, we especially consider a probabilistic representation mapping. Specifically, we use a Gaussian distribution in all of our experiments, e.g., $p_\theta(z|x)=\mathcal{N}(z;\mu_\theta(x),\sigma_\theta(x))$. This is just a design choice that is simple and works well in practice, and our work is not limited by this choice of the distribution. We can use almost any other distribution (with a known parameterization trick to allow for backpropagation). Also, note that the Gaussian representation network is a generalized version of a deterministic network (it becomes a deterministic network when $\sigma_\theta(x) \rightarrow 0 \;\forall x$); therefore, this network choice is not at all restricted when compared to a typical deterministic network.

\vspace{0.15in}

Any downstream task $T$ is defined by a conditional distribution $p_T(y|x)$ where $y \in \mathcal{Y}$ is the label. The joint data distribution of this task is $p_T(x,y)=p(x)p_T(y|x)$. With the representation mapping $p_\theta(z|x)$ learned in advance, we want to learn a classifier $\hat{p}_T(y|z)$ (parameterized by $\omega_T$, which we will omit for notation simplicity) for the task $T$. This is often called an output head that classifies $y$ given $z$.\\
The predictive distribution of $y$ given $x$ for this task $T$ is:
\begin{align}
\mathbb{E}_{p_\theta(z|x)}[\hat{p}_T(y|z)]
\label{predictive}
\end{align}

(for a deterministic representation mapping, Eq.~\ref{predictive} simplifies into $\hat{p}_T(y|z=g_\theta(x))$)

\begin{remark}
    On the inference complexity of a probabilistic representation.
\end{remark}
\vspace{-0.08in}
\quad Using a probabilistic representation, we need to sample multiple $z$ from $p_\theta(z|x)$ to estimate Eq.~\ref{predictive} with Monte Carlo sampling during test time. However, this is not a big issue for the representation learning framework, since we only need to run the representation network $p_\theta(z|x)$ (which is usually deep) once to get a distribution of $z$. After sampling multiple $z$ from that distribution, we only need to rerun the classifier $\hat{p}_T(y|z)$, which is usually a small network (e.g., often contains one or a few fully-connected layers). Furthermore, we can also run $\hat{p}_T(y|z)$ (a small network) in parallel for multiple $z$ to reduce inference time if necessary.

During training of the downstream task $T$, a single $z$ is sampled per input $x$ from the learned representation mapping $p_\theta(z|x)$; and the output head $\hat{p}_T(y|z)$ is trained via minimizing the following training objective:
\begin{equation}
    \mathbb{E}_{p_T(x,y)}\left[\mathbb{E}_{p_\theta(z|x)}[-\log \hat{p}_T(y|z)]\right]
    \label{train_obj}
\end{equation}

With common choices of the predictive distribution $\hat{p}_T(y|z)$, the quantity $-\log$\;$\hat{p}_T(y|z)$ is often non-negative. For example, with a categorical predictive distribution in a classification problem, this term is the cross-entropy loss; whereas with a Gaussian predictive distribution (with a fixed variance) in a regression problem, the above term becomes the squared error loss (with an additive constant).

\makebox[0.96\linewidth][s]{Note also that the objective in Eq.~\ref{train_obj} is an upper bound of the true loss for task $T$}\\$\mathbb{E}_{p_T(x,y)}\left[-\log \mathbb{E}_{p_\theta(z|x)}[\hat{p}_T(y|z)]\right]$ (due to Jensen's inequality), where $-\log \mathbb{E}_{p_\theta(z|x)}[\hat{p}_T(y|z)]$ is the loss of a datapoint $(x,y)$.

Now we formally define the adversarial robustness of the network on the downstream task. Denote $A(x)$ to be the set of adversarial examples of $x$. This set is different for different kinds of adversarial attack; for example, with an $l_\infty$ attack, $A(x)$ is the $l_\infty$-ball around $x$ with a predefined radius of $\epsilon$. The adversarial loss of the task $T$ is:
\begin{align}
\vspace{-0.1in}
\mathbb{E}_{p_T(x,y)}\left[\max_{x^{adv}\in A(x)} -\log \mathbb{E}_{p_\theta(z|x^{adv})}[\hat{p}_T(y|z)]\right]
\vspace{-0.1in}
\end{align}

Intuitively, this means that an attacker seeks to find an adversarial example $x^{adv}$ of each input $x$ that maximizes the loss w.r.t. its label $y$; and we, as the defender, want to minimize that loss. In the next subsections, we will discuss how we can minimize this adversarial loss, even in a task-agnostic manner during the representation learning phase.

\vspace{-0.1in}
\subsection{A bound on the adversarial loss}
We first propose a bound on the adversarial loss based on the downstream training objective and a robustness regularizer:

\begin{proposition} Assuming that $\forall x, \; p_\theta(z|x)$ has the same support set $\mathcal{Z}$ (e.g., $p_\theta(z|x)$ is Gaussian); and that $-\log \hat{p}_T(y|z) \leq M \;\forall z\in\mathcal{Z}, y\in\mathcal{Y}$ \footnote{In the classification problem, we can enforce this quite easily by augmenting the output softmax of the classifier $\hat{p}_T(y|z)$ so that each class probability is always at least $\exp{(-M)}$. For example, if we choose $M=3 \Rightarrow \exp{(-M)}\approx 0.05$, and if the output softmax is $(p_1,p_2,...,p_C)$, we can augment it into $(p_1\cdot K+0.05,p_2\cdot K+0.05,...,p_C\cdot K+0.05)$, where  $K=1-0.05\cdot C$ and $C$ is the number of classes. This ensures the bound for the loss of a datapoint, while retaining the output prediction class.}, we have:
\vspace{-0.1in}
\begin{align}
\mathbb{E}_{p_T(x,y)}\left[\max_{x^{adv}\in A(x)} -\log \mathbb{E}_{p_\theta(z|x^{adv})}[\hat{p}_T(y|z)]\right] \leq\; \mathbb{E}_{p_T(x,y)}\left[\mathbb{E}_{p_\theta(z|x)}[-\log \hat{p}_T(y|z)]\right] \nonumber \\
+ \frac{M}{\sqrt{2}}\sqrt{\mathbb{E}_{p(x)}\left[\max_{x^{u\_adv}\in A(x)} \textup{KL}[p_\theta(z|x)|p_\theta(z|x^{u\_adv})]\right]}
\label{bound}
\end{align}
\end{proposition}
\begin{proof}
\vspace{-0.1in}
provided in the supplementary file.
\end{proof}

The first term $\mathbb{E}_{p_T(x,y)}\left[\mathbb{E}_{p_\theta(z|x)}[-\log \hat{p}_T(y|z)]\right]$ is the downstream training loss in Eq.~\ref{train_obj}, and will be minimized during the training of the output head $\hat{p}_T(y|z)$ for task $T$. 

We call the second term $\mathbb{E}_{p(x)}\left[\max_{x^{u\_adv}\in A(x)} \text{KL}[p_\theta(z|x)|p_\theta(z|x^{u\_adv})]\right]$ a robustness regularizer. Since we do not want to perform adversarial training for the downstream tasks, we want to minimize this term during the representation learning phase. Since this term is label-free and task-independent, if we minimize it during the representation learning phase, it will transfer directly to the downstream task and help minimize the bound in Eq.~\ref{bound}. This will be discussed further in Subsection~\ref{application}. Recall that we use the Gaussian distribution of the per-image representation network, i.e., $p_\theta(z|x)=\mathcal{N}(z;\mu_\theta(x),\sigma_\theta(x))$, so this KL term can be computed analytically (and exactly). Also note that, as discussed in Remark~\ref{gaussian}, the Gaussian representation is a generalized version of a typical deterministic representation, so it is sufficiently expressive for typical Deep Learning problems.

\paragraph{\textbf{Comparison between our bound and TRADES~\cite{zhang2019theoretically}:}}
\begin{itemize}
    \item The bound in TRADES only works for the case of binary classification, while our bound works for the general case of supervised learning (including multi-class classification and regression).
    \item Our robustness regularizer is label-free and task-independent. Therefore, we can minimize it in the representation learning phase (with unsupervised or self-supervised tasks), and it will transfer directly to the downstream tasks. On the other hand, the robustness regularizer in TRADES is task-dependent, thus minimizing the term for a pretext task does not necessarily transfer to the downstream task. Furthermore, TRADES's robustness regularizer requires a predictive distribution of a task to compute, and this might not be applicable to many self-supervised learning methods where there is no prediction task (e.g., contrastive learning). These arguments are also true for almost all existing robustness methods.
\end{itemize}

\paragraph{\textbf{Trade-off between Clean Accuracy and Adversarial Robustness}} The trade-off between a model's performance on clean input and adversarial input has been well observed in practice \cite{zhang2019theoretically}; and this phenomenon can also be explained with our bound. Minimizing the first term in Eq.~\ref{bound} will help the model's performance on clean input, while minimizing the second term increase the model's robustness against adversarial input; and there is an inherent trade-off between them. Minimizing the second term $\mathbb{E}_{p(x)}\left[\max_{x^{u\_adv}\in A(x)} \text{KL}[p_\theta(z|x)|p_\theta(z|x^{u\_adv})]\right]$ too much will compress the representation, hurting its expressiveness and separability among classes. For example, consider the $l_2$ defense with radius $\epsilon$ ($A(x)$ will be the $l_2$-ball around $x$ with radius $\epsilon$). The above regularizer encourages the representation distribution of an input $x$ to be similar to that of its neighbours in the $l_2$-ball. Now, if there exist two inputs $x_1$ and $x_2$ from different classes such that $\epsilon < ||x_1-x_2||_2 < 2\epsilon$, then these two points do not belong to the other's adversarial set. Let $x'=(x_1+x_2)/2$, it follows that $||x_1-x'||_2 = ||x_2-x'||_2 = ||x_1-x_2||_2/2 < \epsilon$, meaning $x' \in A(x_1)$ and $x' \in A(x_2)$. Note that minimizing the regularizer term too much will encourage the representation distribution of both $x_1$ and $x_2$ to be similar to that of $x'$; and the classifier might fail to separate the two datapoints. Therefore, it might hurt the expressiveness and separability of the representation, especially for the inputs around the decision boundary.
\vspace{-0.1in}
\subsection{Applications and Use Cases}
\label{application}
\makebox[\linewidth][s]{In this subsection, we will discuss the use of our robustness regularizer} $\mathbb{E}_{p(x)}\left[\max_{x^{u\_adv}\in A(x)} \text{KL}[p_\theta(z|x)|p_\theta(z|x^{u\_adv})]\right]$. Although this term can be used directly in a supervised learning setting, a far more exiting application is to minimize it during the representation learning phase (of an unsupervised or self-supervised method). As mentioned earlier, since this term is task-independent, if we minimize it during the representation learning phase, it will transfer directly to downtream tasks, improving the model's adversarial robustness on these tasks (see Eq.~\ref{bound}). We name our regularizer \textbf{Urkle} (\textbf{U}nsupervised \textbf{R}obustness with \textbf{KL} divergenc\textbf{E}). In this subsection, we demonstrate some example scenarios to learn a meaningful (and robust) representation, namely with unsupervised learning (via VAE) and self-supervised learning (with any pretext task).

\vspace{-0.1in}
\subsubsection{With VAE}
\label{with_vae}
In VAE \cite{kingma2013auto}, we have an encoder $p_\theta(z|x)$ (which also acts as our representation mapping), a decoder $q_\phi(x|z)$, and a prior $p(z)$, the objective of VAE (negative ELBO) is:
\begin{align}
\mathbb{E}_{p(x)}\left[\mathbb{E}_{p_\theta(z|x)}[-\log q_\phi(x|z)]\right] + \mathbb{E}_{p(x)}\left[\text{KL}[p_\theta(z|x)|p(z)]\right]
\end{align}

Here we add our robustness regularizer to learn a robust encoder $p_\theta(z|x)$, leading to the below objective:
\begin{align}
\mathbb{E}_{p(x)}\left[\mathbb{E}_{p_\theta(z|x)}[-\log q_\phi(x|z)]\right] + \beta_{VAE}\mathbb{E}_{p(x)}\left[\text{KL}[p_\theta(z|x)|p(z)]\right] \nonumber\\ 
+ \beta_{robust} \mathbb{E}_{p(x)}\left[\max_{x^{u\_adv}\in A(x)} \text{KL}[p_\theta(z|x)|p_\theta(z|x^{u\_adv})]\right]
\label{obj_vae}
\end{align}

Note that we also add a coefficient $\beta_{VAE}$ for the VAE's regularizer
$\mathbb{E}_{p(x)}\left[\text{KL}[p_\theta(z|x)|p(z)]\right]$ (similar to $\beta$-VAE \cite{higgins2016beta}). We use a small value of $\beta_{VAE}$ in practice since we found that this term might hinder the expressiveness of the representation $p_\theta(z|x)$.

All three expectation terms in Eq~\ref{obj_vae} can be estimated with a minibatch of input $x$'s. For each input $x$, we find the unsupervised adversary $x^{u\_adv}$ by the inner maximization problem of $\max_{x^{u\_adv}\in A(x)} \text{KL}[p_\theta(z|x)|p_\theta(z|x^{u\_adv})]$ with, for example, the PGD algorithm. The objective in Eq~\ref{obj_vae} is minimized with respect to $\theta$ and $\phi$ (thus it will learn an encoder $p_\theta(z|x)$ such that the representation of an input is close to that of its adversaries).

\vspace{-0.1in}
\subsubsection{With a self-supervised learning task}
Let's assume we have a pretext task designed to learn a meaningful representation of $x$. Since self-supervised tasks and their loss functions are diverse, we will refer to the loss function as $L_{ssl}(p(x),p_\theta(z|x),\phi)$ in general, where $p_\theta(z|x)$ is a representation mapping used to solve the task (e.g., might be used for a pretext classification task, to solve a jigsaw puzzle, or to minimize the contrastive loss in the contrastive learning framework) and $\phi$ is any additional parameters (apart from $\theta$) used for this self-supervised task (e.g., parameters of the projector in SimCLR \cite{chen2020simple}, parameters of the output head of some pretext classification task).

Similarly, we can also add the robustness regularizer term here to to learn a robust representation mapping $p_\theta(z|x)$, leading to the following objective:
\begin{align}
 L_{ssl}(p(x),p_\theta(z|x),\phi) + \beta_{robust} \mathbb{E}_{p(x)}\left[\max_{x^{u\_adv}\in A(x)} \text{KL}[p_\theta(z|x)|p_\theta(z|x^{u\_adv})]\right]
\end{align}

\paragraph{\textbf{Demonstration with SimCLR \cite{chen2020simple}:}} We re-emphasize that our proposed robustness regularizer can be applied to almost all unsupervised and self-supervised representation learning methods. However, since our main baseline \cite{kim2020adversarial} is built upon SimCLR, we also use SimCLR (with a slight adaptation of using a probabilistic representation network) in our experiments for a fair comparison (note that we conjecture using more recent and advanced self-supervised learning methods \cite{zbontar2021barlow,grill2020bootstrap} will likely improve further the performance of our model). We demonstrate the resulting method here in details. 

First of all, SimCLR (and other contrastive learning methods) learns from two augmentations of each image. Let denote a minibatch of data as $\{(x_{1,1},x_{1,2})$, $(x_{2,1},x_{2,2}),...,(x_{b,1},x_{b,2})\}$, where $x_{i,1}$ and $x_{i,2}$ are two augmented version of a same original image $x_i$, and are called positive examples of each other. The goal of contrastive learning is to encourage the representation to be similar among the positive images. For each $x$ we sample a single $z$ from the network $p_\theta(z|x)$ to make the minibatch of representation $\{(z_{1,1},z_{1,2}),(z_{2,1},z_{2,2}),...,(z_{b,1},z_{b,2})\}$. Then the loss function of SimCLR is:
\begin{align}
\ell_{SimCLR} = \ell_{NT\_Xent}((z_{1,1},z_{1,2}),(z_{2,1},z_{2,2}),...,(z_{b,1},z_{b,2}))
\end{align}
where $\ell_{NT\_Xent}$ is the so-called ``normalized temperature-scaled cross entropy'' loss function that takes input as a batch of tuples $t_1,t_2,...,t_b$, with each $t_i$ is a tuple of representations from a set of positive images (in the case above each tuple would be of length 2). Specifically:
\begin{align}
&\ell_{NT\_Xent}(t_1,t_2,...,t_b) = \sum_{i=1}^b \sum_{z,z'\in t_i,z\neq z'}\frac{\exp(\mathrm{sim}(z,z')/\tau)}{\sum_{j\neq i}\sum_{z''\in t_j} \exp(\mathrm{sim}(z,z'')/\tau)}
\label{nt_xent}
\end{align}
where $\mathrm{sim}(z,z')$ is the cosine similarity between $z$ and $z'$ (we also often project the representation $z$ to a lower dimensional space before calculating the cosine similarity), and $\tau$ is the temperature (often set to $0.5$).

To implement our regularizer, we find the unsupervised adversaries\\ $\{(x^{u\_adv}_{1,1},x^{u\_adv}_{1,2}),...,(x^{u\_adv}_{b,1},x^{u\_adv}_{b,2})\}$ of $\{(x_{1,1},x_{1,2}),...,(x_{b,1},x_{b,2})\}$ that maximize:
\begin{align}
    \frac{1}{2b}\sum_{i=1}^b \sum_{k=1}^2 \max_{x^{u\_adv}_{i,k} \in A(x_{i,k})} \text{KL}[p_\theta(z|x_{i,k}))|p_\theta(z|x^{u\_adv}_{i,k}))]
\end{align}
Following \cite{kim2020adversarial}, we also use PGD \cite{madry2017towards} to solve this inner maximization problem.

Let $\ell_{Urkle}$ be the value of the above maximum, i.e.:
\begin{align}
    \ell_{Urkle} = \frac{1}{2b}\sum_{i=1}^b \sum_{k=1}^2 \text{KL}[p_\theta(z|x_{i,k}))|p_\theta(z|x^{u\_adv}_{i,k}))]
\end{align}
with $\{x^{u\_adv}_{i,k}\}$ found above.

Now we can add the regularizer $\ell_{Urkle}$ directly to the original loss function $\ell_{SimCLR}$. However, we note that $x^{u\_adv}_{i,1}$ and $x^{u\_adv}_{i,2}$ can also be treated as positive images of $x_{i,1}$ and $x_{i,2}$; therefore, we also include them when compute the $NT\_Xent$ loss. The final loss function of our model as:
\begin{align}
    &\ell_{NT\_Xent}((z_{i,1},z_{i,2},z^{u\_adv}_{i,1},z^{u\_adv}_{i,2})_{i=1}^b)  + \beta_{robust} \ell_{Urkle}
\end{align}
where $z^{u\_adv}_{i,k} \sim p_\theta(z|x^{u\_adv}_{i,k}) \;\forall i \in \overline{1,b}, k \in \{1,2\}$ , $\ell_{NT\_Xent}$ is calculated as in Eq~\ref{nt_xent} (in this case each tuple is of length 4), and $\beta_{robust}$ is a hyperparameter.
\vspace{-0.1in}
\paragraph{\textbf{PGD attack/defense:}} Since Projected Gradient Descent is used in both training and evaluation in our experiments, we briefly review it here to make the paper more self-contained. Let's assume that we need to find $x'$ within an $l_\infty$-norm (or other norms) ball of radius $\epsilon$ around $x$ that maximizes the function $f(x')$. The PGD algorithm is as follows:
\begin{enumerate}
    \item Initialize $x^0$ to $x$ (possibly with a small added random noise).
    \item Update $x^i=x^{i-1} + \alpha \texttt{sign}(\nabla_{x^{i-1}}f(x^{i-1}))$ with a step size $\alpha$, and clip $x^i$ to the $\epsilon$-ball that is being considered.
    \item Repeat step 2 $k$ times, and set $x' = x^k$.
\end{enumerate}

Note that we can also use any other adversarial attack/defense methods for the inner maximization problem. We use PGD in this paper because it is considered ``gold-standard'' at the moment, and it is also used by our main baseline \cite{kim2020adversarial}.

%% file: sections/experiments.tex
\vspace{-0.1in}
\section{Experiments}
\vspace{-0.1in}
We conduct extensive experiments to validate our method. In this section, we describe these experiments in details. For more information regarding the experimental settings and the baselines, please refer to our supplementary file.
\vspace{-0.25in}
\subsection{Datasets}
\vspace{-0.05in}
\paragraph{\textbf{MNIST \cite{lecun-mnisthandwrittendigit-2010}}} contains 70000 images of hand-written digits with the classification task of 10 digits.
\vspace{-0.07in}
\paragraph{\textbf{CIFAR10 \cite{krizhevsky2009learning}}} consists of 60000 images of size 32x32, and over ten classes: airplane, automobile, bird, cat, deer, dog, frog, horse, ship and truck.
\vspace{-0.07in}
\paragraph{\textbf{CIFAR100 \cite{krizhevsky2009learning}}} Similar to CIFAR10, CIFAR100 also consists of 60000 images with the size 32x32. The task is classification with 100 different classes.
\vspace{-0.07in}
\paragraph{\textbf{ImageNet \cite{deng2009imagenet}}} is a large scale real-world computer vision dataset, which consists of 1000 classes. To the best of our knowledge, this dataset has not been considered by previous adversarial self-supervised learning methods. 

\vspace{-0.1in}
\subsection{Experimental Settings}
Within each experiment, we use the same network as the representation network for all models. Since our representation network is probabilistic, it only differs from the other deterministic networks in the last layer. In particular, for a representation of size $d_z$, the last layer's dimension of a deterministic representation network is $d_z$, while that of a probabilistic network is $2\cdot d_z$ ($d_z$ for $\mu$ and $d_z$ for $\sigma^2$). Also, although our method can be used for any adversarial attack (e.g., $l_1$ or $l_2$), we consider the $l_\infty$ adversaries in this experiments section. This is because $l_\infty$ is one of the most common adversarial attacks, and it is also used in our main baseline \cite{kim2020adversarial}.
\vspace{-0.05in}
\paragraph{\textbf{With VAE:}} We test the effectiveness of our robustness regularizer when used with VAE (as described in Section~\ref{with_vae}) with the MNIST dataset. The main baseline we consider in this experiment is AE (auto-encoder) with a TRADES-like regularizer. This is because the reconstruction in AE can be viewed as a prediction task, so we can use the regularizer in TRADES to force the reconstruction of an image to be similar to the reconstruction of its adversaries (more details of this baseline in the supplementary file). To make the comparison fair for AE (that has no built-in regularizer), we set $\beta_{VAE}=0$ in our experiment, although we note that slightly increasing this value leads to even better representation. Apart from this baseline, we also include supervised learning models (Standard Training, AT \cite{madry2017towards} and TRADES \cite{zhang2019theoretically}) for reference. For this experiment, we consider the $l_\infty$ perturbation with $\epsilon=0.1$. For this ``toy'' experiment, we use a simple convolutional neural network with four 3$\times$3 convolutional layers (followed by an average pooling layer) as the representation network.

\vspace{-0.05in}
\paragraph{\textbf{With Self-Supervised Learning (SimCLR):}} For the more challenging real-world datasets (CIFAR10, CIFAR100, ImageNet), learning generative features of images with VAE is difficult, so we use self-supervised learning methods (SimCLR) to validate our robustness regularizer. In this experiment, we consider RoCL \cite{kim2020adversarial} as our main baseline. We also include supervised learning models (Standard Training, AT \cite{madry2017towards} and TRADES \cite{zhang2019theoretically}) for reference.

We train our self-supervised model with 2000 epochs (except for ImageNet, where we train a total number of 200 epochs for computational reasons). Similar to \cite{kim2020adversarial}, we consider the $l_\infty$ attack and defense. Following standard in the adversarial robustness literature, we set the training perturbation radius as $\epsilon=8/255$ for CIFAR10/CIFAR100 and $\epsilon=2/255$ for ImageNet.

For CIFAR10, we use ResNet18 \cite{he2016deep} as the backbone network with a batchsize of 1024. As for the CIFAR100 dataset, we use ResNet50 as the backbone, and also with a batchsize of 1024. For ImageNet, we use a batchsize of 4096 with a ResNet50 network. For all experiments, we use a starting learning rate of $1.2$ and perform Cosine annealing on the learning rate over the course of training.
\vspace{-0.1in}
\subsection{Results}
\subsubsection{With VAE}

\begin{table*}[t!]
\centering
\captionsetup{justification=centering}
	\caption{\textbf{MNIST} with $l_\infty$ adversaries. Training and testing $\epsilon$ are set to $0.1$ in this experiment. Our method (VAE+Urkle) outperforms the baseline AE+TRADES, while approaching the robustness similar to supervised adversarial training methods.}
 	\label{mnist}
		\begin{tabular}{ccc}
			\toprule
			Models & Clean Acc & Adversarial Acc \\
			\midrule
			Standard Training  & 99.3±0.1 & 1.0±0.3 \\
			AT  & 99.0±0.1 & \textbf{98.5±0.2} \\
			TRADES  & 99.1±0.1 & 98.3±0.1  \\
    		\midrule
			AE + TRADES  & 99.1±0.1 & 96.6±0.4 \\
			VAE + Urkle (ours)  & 99.1±0.1 & \textbf{98.0±0.1} \\
			\bottomrule
		\end{tabular}
\end{table*}

\begin{table*}[t!]
 \centering
 \captionsetup{justification=centering}
 \caption{ \textbf{CIFAR10}: Results of supervised and self-supervised methods trained with $l_\infty$ adversaries and $\epsilon=8/255$ (when applicable). Our method (SimCLR+Urkle) significantly outperforms the baseline RoCL, especially with unseen (and stronger) attack $\epsilon=16/255$.}
  \label{cifar10}
 \centering
 \small
 	\begin{tabular}{ccccccccccc}
 		\toprule
        & \multicolumn{9}{c}{CIFAR10}\\
        \cmidrule(r){2-10}
		& \multicolumn{3}{c}{Fully Labeled Data} & \multicolumn{3}{c}{5000 Labeled Data} & \multicolumn{3}{c}{1000 Labeled Data}\\
		\cmidrule(r){2-4} \cmidrule(r){5-7} \cmidrule(r){8-10}
		Model  & Clean &  8/255  & 16/255  & Clean &  8/255  & 16/255 &  Clean &  8/255  & 16/255 \\
		\midrule
		Standard Training & 92.82 & 0.00  & 0.00 & 79.09 & 0.00 & 0.00 & 60.39 & 0.00 & 0.00 \\
		AT  & 81.63 & 44.50 & 14.47 & 64.97 & 24.52 & 6.69 & 50.03 & 15.26 & 3.91 \\
		TRADES  & 77.03 & \textbf{48.01}  & \textbf{22.55} & 63.14 & \textbf{25.97} & \textbf{7.78} & 48.32 & \textbf{15.92} & \textbf{3.97} \\
		\midrule
		SimCLR  & 91.25 & 0.63 & 0.15 & 84.31 & 0.84 & 0.12 & 82.15 & 0.55 & 0.11 \\
		RoCL  & 83.71 & 40.27 & 9.55 & 78.82 & 36.93 & 9.90 & 76.49 & 34.44 & 8.96 \\
		SimCLR+Urkle (ours) & 82.31 & \textbf{42.56} & \textbf{14.29} & 77.47 & \textbf{38.76} & \textbf{12.94} & 74.82 & \textbf{37.56} & \textbf{12.22} \\
 		\bottomrule
 	\end{tabular}
\end{table*}

Table~\ref{mnist} shows the results of our model and the baselines. MNIST is a relatively easy dataset, so most methods perform reasonably well. Noticeably, our model (VAE+Urkle) outperforms AE+TRADES by 1.4\%, and approaches the performance of supervised adversarial training methods.

\vspace{-0.1in}
\subsubsection{With Self-Supervised Learning (SimCLR):} \;
\vspace{-0.1in}
\paragraph{\textbf{CIFAR10 and CIFAR100:}} As aforementioned, in the CIFAR10 and CIFAR100 experiments, we train all models (except for Standard Training and SimCLR) with $\epsilon = 8/255$. We evaluate these models against $l_\infty$ adversarial attack with strength $\epsilon=8/255$ and $\epsilon=16/255$. Table~\ref{cifar10} and Table~\ref{cifar100} show that our method (SimCLR+Urkle) clearly outperforms RoCL (with a slight trade-off of clean accuracy in some experiments), indicating the effectiveness of our robustness regularizer. Especially, our method is more robust against unseen attack strength ($\epsilon$=16/255).

As discussed earlier, there is a trade-off between the clean accuracy and adversarial robustness in our model (as well as other adversarial training methods such as TRADES), resulting in a slightly lower (around 1\%) clean accuracy of our model when compared to RoCL. However, our model outperforms RoCL significantly in terms of adversarial robustness (especially for $\epsilon=16/255$, which is an unseen attack strength). We find that this is reasonable and the improved robustness is well worth the trade-off. Note that a similar trend can be observed for TRADES and AT, where TRADES achieves lower clean accuracy but much better adversarial robustness when compared to AT.
\begin{table*}[t!]
 \centering
 \captionsetup{justification=centering}
 \caption{\textbf{CIFAR100}: esults of supervised and self-supervised methods trained with $l_\infty$ adversaries and $\epsilon=8/255$ (when applicable). Our method (SimCLR+Urkle) significantly outperforms the baseline RoCL.}
  \label{cifar100}
 \centering
 \small
 	\begin{tabular}{ccccccccccc}
 		\toprule
        & \multicolumn{9}{c}{CIFAR100}\\
        \cmidrule(r){2-10}
		& \multicolumn{3}{c}{Fully Labeled Data} & \multicolumn{3}{c}{5000 Labeled Data} & \multicolumn{3}{c}{1000 Labeled Data}\\
		\cmidrule(r){2-4} \cmidrule(r){5-7} \cmidrule(r){8-10}
		Model  & Clean &  8/255  & 16/255  & Clean &  8/255  & 16/255 &  Clean &  8/255  & 16/255 \\
		\midrule
		Standard Training & 70.34 & 0.00 & 0.00 & 26.59 & 0.00 & 0.00 & 12.14 & 0.00 & 0.00 \\
		AT & 52.87 & \textbf{19.46} & \textbf{6.80} & \textbf{21.05} & 5.30 & 1.52 & 12.46 & \textbf{3.26} & 0.95 \\
		TRADES & 56.96 & 18.54 & 4.48 & 20.35 & \textbf{6.41} & \textbf{1.63} & 13.78 & 3.19 & \textbf{1.32} \\
		\midrule
		SimCLR  & 58.79 & 0.47 & 0.00 & 53.23 & 0.46 & 0.12 & 44.16 & 0.57 & 0.26 \\
		RoCL  & 52.19 & 22.00 & 8.35 & 40.60 & 18.86 & 7.25 & 30.23 & 13.83 & 5.55 \\
		SimCLR+Urkle (ours) & 53.81 & \textbf{24.82} & \textbf{10.63} & 42.66 & \textbf{22.34} & \textbf{11.06} & 32.12 & \textbf{14.40} & \textbf{6.71} \\
 		\bottomrule
 	\end{tabular}
\end{table*}

\begin{table}[t!]
	\caption{\textbf{ImageNet} with $l_\infty$ adversaries and $\epsilon=2/255$.}
 	\label{imagenet}
	\centering
	\resizebox{0.7\columnwidth}{!}{
		\begin{tabular}{c|cc}
			\toprule
			Models & Clean Acc & Adversarial Acc \\
			\midrule
			RoCL  & 52.46 & 23.19 \\
			SimCLR + Urkle (ours)  & 51.19 & \textbf{25.69} \\
			\bottomrule
		\end{tabular}
}
\vspace{-0.1in}
\end{table}
\vspace{-0.1in}
\paragraph{\textbf{ImageNet:}} Preliminary result on the ImageNet dataset (Table~\ref{imagenet}) also indicates that our method outperforms RoCL on this large scale dataset.
\vspace{-0.1in}
\paragraph{\textbf{Experimental Results with limited numbers of labels:}} To take advantage of the unsupervised nature of our method, we also conduct the experiments when the number of labeled images is limited. To be re-emphasize, we train our SSL model and the unsupervised robustness regularizer \textbf{without any labels}, and the limited number of labels are only used for the training of the task-specific output head $\hat{p}(y|z)$ (without adversarial training). Table~\ref{cifar10} and Table~\ref{cifar100} report the results for CIFAR10 and CIFAR100 with 5000 and 1000 labels. It can be clearly seen that supervised methods fail to learn a robust model with such a few available labels. Among the unsupervised robust representation learning methods, our model also significantly outperforms the baseline RoCL in these scenarios. 

%% file: sections/conclusion.tex
\section{Conclusion}
To conclude, in this paper, we develop a task-agnostic robust representation learning method. The core idea behind our method is to minimize a task-independent robustness regularizer that enforces the representation of an image to be close to that of its adversarial examples. This is motivated by our theoretical result that, for a model using the learned representation for a downstream task, its adversarial loss is bounded by the loss on clean image plus the above task-independent regularizer. Our regularizer can be straightforwardly applied to almost any existing representation learning method (with only an adaptation to a probabilistic representation). To the best of our knowledge, our work is one of the first to study the problem of unsupervised robust representation learning in a principled way, and show that the robustness can be theoretically transferred to the downstream tasks. We demonstrate our proposed regularizer with several unsupervised/self-supervised methods (from VAE to SimCLR), and conduct extensive experiments on MNIST, CIFAR10, CIFAR100 and ImageNet to validate our method. Experimental results suggest that our proposed method (when used with SimCLR) achieves SOTA performance on the unsupervised robust representation learning task.

%% file: sections/appendix.tex
\appendix 
\section{Proofs}
\subsection{Proposition 1}
We use a similar idea as the proof of Proposition 1 in \cite{nguyen2021kl}, which is a bound on the target loss in the domain adaptation problem. This is because our setting can somewhat be cast as a domain adaptation problem, where the source domain is the clean data distribution, and the target domain is the adversarial data distribution. The proof is as below:
\begin{proof} $ $\newline
Let $a(x) = \arg \max_{x^{adv}\in A(x)} -\log \mathbb{E}_{p_\theta(z|x^{adv})}[\hat{p}_T(y|z)]$ $\forall x \in \mathcal{X}$. We need to prove that:

\begin{align}
\mathbb{E}_{p_T(x,y)}\left[-\log \mathbb{E}_{p_\theta(z|a(x))}[\hat{p}_T(y|z)]\right] 
\leq \mathbb{E}_{p_T(x,y)}\left[\mathbb{E}_{p_\theta(z|x)}[-\log \hat{p}_T(y|z)]\right] \nonumber \\
+ \frac{M}{\sqrt{2}}\sqrt{\mathbb{E}_{p(x)}\left[\max_{x^{u\_adv}\in A(x)} \text{KL}[p_\theta(z|x)|p_\theta(z|x^{u\_adv})]\right]}
\end{align}

Due to Jensen Inequality, we have:
\begin{align}
\mathbb{E}_{p_T(x,y)}\left[-\log \mathbb{E}_{p_\theta(z|a(x))}[\hat{p}_T(y|z)]\right] 
\leq \mathbb{E}_{p_T(x,y)}\left[\mathbb{E}_{p_\theta(z|a(x))}[-\log \hat{p}_T(y|z)]\right] 
\end{align}

Therefore, we only need to prove that:
\begin{align}
&\mathbb{E}_{p_T(x,y)}\left[\mathbb{E}_{p_\theta(z|a(x))}[-\log \hat{p}_T(y|z)]\right]
- \mathbb{E}_{p_T(x,y)}\left[\mathbb{E}_{p_\theta(z|x)}[-\log \hat{p}_T(y|z)]\right] \nonumber \\
\leq  &\frac{M}{\sqrt{2}}\sqrt{\mathbb{E}_{p(x)}\left[\max_{x^{u\_adv}\in A(x)} \text{KL}[p_\theta(z|x)|p_\theta(z|x^{u\_adv})]\right]}
\end{align}

We have:
\begin{align}
&\mathbb{E}_{p_T(x,y)}\left[\mathbb{E}_{p_\theta(z|a(x))}[-\log \hat{p}_T(y|z)]\right]
- \mathbb{E}_{p_T(x,y)}\left[\mathbb{E}_{p_\theta(z|x)}[-\log \hat{p}_T(y|z)]\right] \\
= &\mathbb{E}_{p_T(x,y)}\left[\int_\mathcal{Z} p_\theta(z|a(x)) [-\log \hat{p}_T(y|z)] dz - \int_\mathcal{Z} p_\theta(z|x) [-\log \hat{p}_T(y|z)] dz \right]\\
= &\mathbb{E}_{p_T(x,y)}\left[\int_\mathcal{Z} -\log \hat{p}_T(y|z) \left[p_\theta(z|a(x))-p_\theta(z|x)\right] dz\right]
\end{align}

For all $x$, let $\mathcal{A}(x)=\{z \in \mathcal{Z}|p_\theta(z|a(x))-p_\theta(z|x) \geq 0\}$ and $\mathcal{B}(x)=\{z \in \mathcal{Z}|p_\theta(z|a(x))-p_\theta(z|x) < 0\}$, then:
\begin{align}
& \mathbb{E}_{p_T(x,y)}\left[\int_\mathcal{Z} -\log \hat{p}_T(y|z) \left[p_\theta(z|a(x))-p_\theta(z|x)\right] dz\right] \\
= &\mathbb{E}_{p_T(x,y)}\Bigg[\int_{\mathcal{A}(x)} -\log \hat{p}_T(y|z) \left[p_\theta(z|a(x))-p_\theta(z|x)\right] dz \nonumber\\ 
&\quad\quad\quad+ \int_{\mathcal{B}(x)} -\log \hat{p}_T(y|z) \left[p_\theta(z|a(x))-p_\theta(z|x)\right] dz\Bigg] \\
\leq &\mathbb{E}_{p_T(x,y)}\Bigg[\int_{\mathcal{A}(x)} -\log \hat{p}_T(y|z) \left[p_\theta(z|a(x))-p_\theta(z|x)\right] dz\Bigg] \\
& \quad\quad \text{(since $-\log \hat{p}_T(y|z)$ is a non-negative quantity)} \nonumber\\
\leq &\mathbb{E}_{p_T(x,y)}\left[\int_{\mathcal{A}(x)} M \left[p_\theta(z|a(x))-p_\theta(z|x)\right] dz\right] \\
= &M \mathbb{E}_{p_T(x,y)}\left[\int_{\mathcal{A}(x)} \left|p_\theta(z|a(x))-p_\theta(z|x)\right| dz\right]
\end{align}

We have:
\begin{align}
&\int_{\mathcal{Z}} \left[p_\theta(z|a(x))-p_\theta(z|x)\right] dz = 0 \\
\Rightarrow & \int_{\mathcal{A}(x)} \left[p_\theta(z|a(x))-p_\theta(z|x)\right] dz
+ \int_{\mathcal{B}(x)} \left[p_\theta(z|a(x))-p_\theta(z|x)\right] dz = 0 \\
\Rightarrow & \int_{\mathcal{A}(x)} \left|p_\theta(z|a(x))-p_\theta(z|x)\right| dz 
= \int_{\mathcal{B}(x)} \left|p_\theta(z|a(x))-p_\theta(z|x)\right| dz \\
\Rightarrow & \int_{\mathcal{A}(x)} \left|p_\theta(z|a(x))-p_\theta(z|x)\right| dz 
= \frac{1}{2}\int_{\mathcal{Z}} \left|p_\theta(z|a(x))-p_\theta(z|x)\right| dz 
\end{align}

Due to the Pinsker's Inequality we have:
\begin{align}
&\frac{1}{2}\int_{\mathcal{Z}} \left|p_\theta(z|a(x))-p_\theta(z|x)\right| dz \nonumber\\
\leq &\frac{1}{2}\sqrt{2\int_\mathcal{Z}p_\theta(z|x)\log \frac{p_\theta(z|x)}{p_\theta(z|a(x))} dz} \\
= &\frac{1}{\sqrt{2}}\sqrt{\text{KL}[p_\theta(z|x)|p_\theta(z|a(x))]}
\end{align}

Therefore:
\begin{align}
& M \mathbb{E}_{p_T(x,y)}\left[\int_{\mathcal{A}(x)} \left|p_\theta(z|a(x))-p_\theta(z|x)\right| dz\right] \\
\leq &\frac{M}{\sqrt{2}} \mathbb{E}_{p_T(x,y)}\left[\sqrt{\text{KL}[p_\theta(z|x)|p_\theta(z|a(x))]}\right] \\
=& \frac{M}{\sqrt{2}} \mathbb{E}_{p(x)}\left[\sqrt{\text{KL}[p_\theta(z|x)|p_\theta(z|a(x))]}\right] \\
\leq& \frac{M}{\sqrt{2}} \sqrt{\mathbb{E}_{p(x)}\left[\text{KL}[p_\theta(z|x)|p_\theta(z|a(x))]\right]} \\
\leq &\frac{M}{\sqrt{2}} \sqrt{\mathbb{E}_{p(x)}\left[\max_{x^{u\_adv}\in A(x)}\text{KL}[p_\theta(z|x)|p_\theta(z|x^{u\_adv})]\right]}
\end{align}

We conclude our proof.
\end{proof}

\section{Experimental Results}

\subsection{Details on the baseline AE+TRADES}

Here we describe the baseline AE+TRADES, which we use in the VAE experiment, in more detail. In particular, an autoencoder (AE) consists of an encoder $g_\theta$ and a decoder $h_\phi$. The encoder $g$ transforms the input $x$ to a representation $z$ (often lower dimensional), i.e., $z=g_\theta(x)$; while the decoder $h$ tries to reconstruct the original input from the representation, i.e., $\hat{x}=h_\phi(z)$. Using the mean squared (l2) distance for the reconstruction, the objective of AE is:

\begin{align}
	\mathbb{E}_{p(x)}[||x-h_\phi(g_\theta(x))||_2^2]
\end{align}

Since the reconstruction $h\circ g$ can be treated as a prediction task (predicting the original $x$), we can use a TRADES-like regularizer to make the model more robust, and thus the encoder is also more robust. Specially, we can enforce the reconstruction of an image to be similar to that of its adversaries. The final objective is:

\begin{align}
	&\mathbb{E}_{p(x)}[||x-h_\phi(g_\theta(x))||_2^2] \nonumber\\
	& + \beta \mathbb{E}_{p(x)}[\max_{x^{adv}\in A(x)}||h_\phi(g_\theta(x^{adv}))-h_\phi(g_\theta(x))||_2^2]
\end{align}

Note that we can only use this baseline with AE (not VAE) because the TRADES regularizer only works straightforwardly with a deterministic model / prediction.

\subsection{Experimental Settings}
\subsubsection{With VAE}
For the VAE MNIST experiment, the encoder (representation network) is a simple convolutional network with 4 \texttt{3x3} convolutional layers (with the last layer has 128 channels so that the representation has 128 dimension), followed by an average pooling layer.

With AE and VAE, the decoder consists of 4 ConvTranspose2d layers, mirroring the encoder.

The classifier (from a representation to the prediction label) is a composition of 3 fully connected layers (with batchnorm and ReLU activation in-between).

\subsubsection{With SimCLR}
For a ResNet18 backbone network, we set the representation dimension to 512, while that for a ResNet50 backbone is 2048. We set the initial learning rate to $1.2$ and do Cosine annealing to $0$. Other experiment details have been presented in the main paper. In addition, please also refer to our code for more details.